\documentclass[letterpaper, 10 pt, journal, twoside]{IEEEtran}  

\IEEEoverridecommandlockouts                              

\usepackage{amsmath} 
\usepackage{amsfonts}
\usepackage{amssymb}  
\usepackage{dsfont}
\usepackage{graphicx}
\graphicspath{{figures/}}
\usepackage{subcaption}
\usepackage{array}
\usepackage{multirow}
\usepackage{hyperref}
\usepackage{cite}
\usepackage{comment}
\usepackage[ruled,linesnumbered]{algorithm2e}
\usepackage{amsthm}
\usepackage[utf8]{inputenc}
\usepackage[english]{babel}
\usepackage{siunitx}
\usepackage{tikz}
\usetikzlibrary{positioning}
\usetikzlibrary{arrows}
\usetikzlibrary{fit}

\DeclareMathOperator{\reals}{\mathbb{R}}
\DeclareMathOperator*{\argmin}{\arg\!\min}
\DeclareMathOperator{\var}{var}
\DeclareMathOperator{\tr}{tr}
\newtheorem{lemma}{Lemma}
\newtheorem{cor}{Corollary}
\newtheorem{rmk}{Remark}
\usepackage{xcolor}

\title{\Large \bf
On-Demand Trajectory Predictions For \\ Interaction Aware Highway Driving
}

\author{Cyrus~Anderson$^{1}$, Ram~Vasudevan$^{2}$, and Matthew~Johnson-Roberson$^{3}$
\thanks{This work was supported by a grant from Ford Motor Company via the Ford-UM Alliance under award N022884.}
\thanks{$^1$C. Anderson is with the Robotics Institute, University of Michigan, Ann Arbor, MI 48109 USA {\tt\footnotesize andersct@umich.edu}}%
\thanks{$^2$M. Johnson-Roberson is with the Department of Naval Architecture and Marine Engineering, University of Michigan, Ann Arbor, MI 48109 USA {\tt\footnotesize mattjr@umich.edu}}%
\thanks{$^3$R. Vasudevan is with the Department of Mechanical Engineering, University of Michigan, Ann Arbor, MI 48109 USA {\tt\footnotesize ramv@umich.edu}}%
}

\begin{document}
\maketitle

\begin{abstract}
Highway driving places significant demands on human drivers and autonomous vehicles (AVs) alike due to high speeds and the complex interactions in dense traffic.
Merging onto the highway poses additional challenges by limiting the amount of time available for decision-making.
Predicting others' trajectories accurately and quickly is crucial to safely executing maneuvers.
Many existing prediction methods based on neural networks have focused on modeling interactions to achieve better accuracy while assuming the existence of observation windows over \SI{3}{s} long.
This paper proposes a novel probabilistic model for trajectory prediction that performs competitively with as little as \SI{400}{\milli\second} of observations.
The proposed model extends a deterministic car-following model to the probabilistic setting by treating model parameters as unknown random variables and introducing regularization terms.
A realtime inference procedure is derived to estimate the parameters from observations in this new model.
Experiments on dense traffic in the NGSIM dataset demonstrate that the proposed method achieves state-of-the-art performance with both highly constrained and more traditional observation windows.
\end{abstract}

\begin{IEEEkeywords}
Autonomous Vehicle Navigation, Autonomous Agents
\end{IEEEkeywords}

\section{Introduction}

Merging in dense traffic necessitates cooperating with other drivers.
Successful cooperation in turn relies on predicting others' actions.
Predicting a vehicle's trajectory, however, is complicated by its possible interactions with surrounding vehicles~\cite{lawitzky2013interactive,lefevre2014survey}.
Recent works based on deep neural networks (DNNs) have proven effective at modeling these interactions~\cite{kim2017nn,xin2018nn,deo2018nn,hu2018nn,chandra2019traphic,gupta2018socialgan,ma2019trafficpredict,zhao2019multi,lee2017desire}.
These methods utilize a fixed number of observations of surrounding vehicles to infer which trajectories are likely.
A drawback to this is the duration of time needed to collect observations before making predictions, ranging from \SI{3}{\second} up to \SI{5}{\second}.
This window of time determines the minimum delay between first seeing a vehicle and predicting its trajectory.
While observation windows of \SI{3}{\second} may be tenable in low speed environments, the high speeds in highway driving call for faster reaction times.
Occlusions and sensor limitations such as maximum range also impact the quality of any observations of the surrounding vehicles.
This motivates the need for predictions that can be made in short time and with few observations.
In this work we aim to strike a balance between the richness of interactions modeled and the number of observations needed to make predictions.
Merging onto the highway in particular poses a challenge to autonomous vehicles (AVs).
In addition to the delay induced by observation windows, limited ramp length will further restrict the time available to make decisions.
Due to the heightened need of fast reaction times when merging onto the highway, we focus on these scenarios.
\begin{figure}[t]
  \centering
  \includegraphics[width=0.45\textwidth]{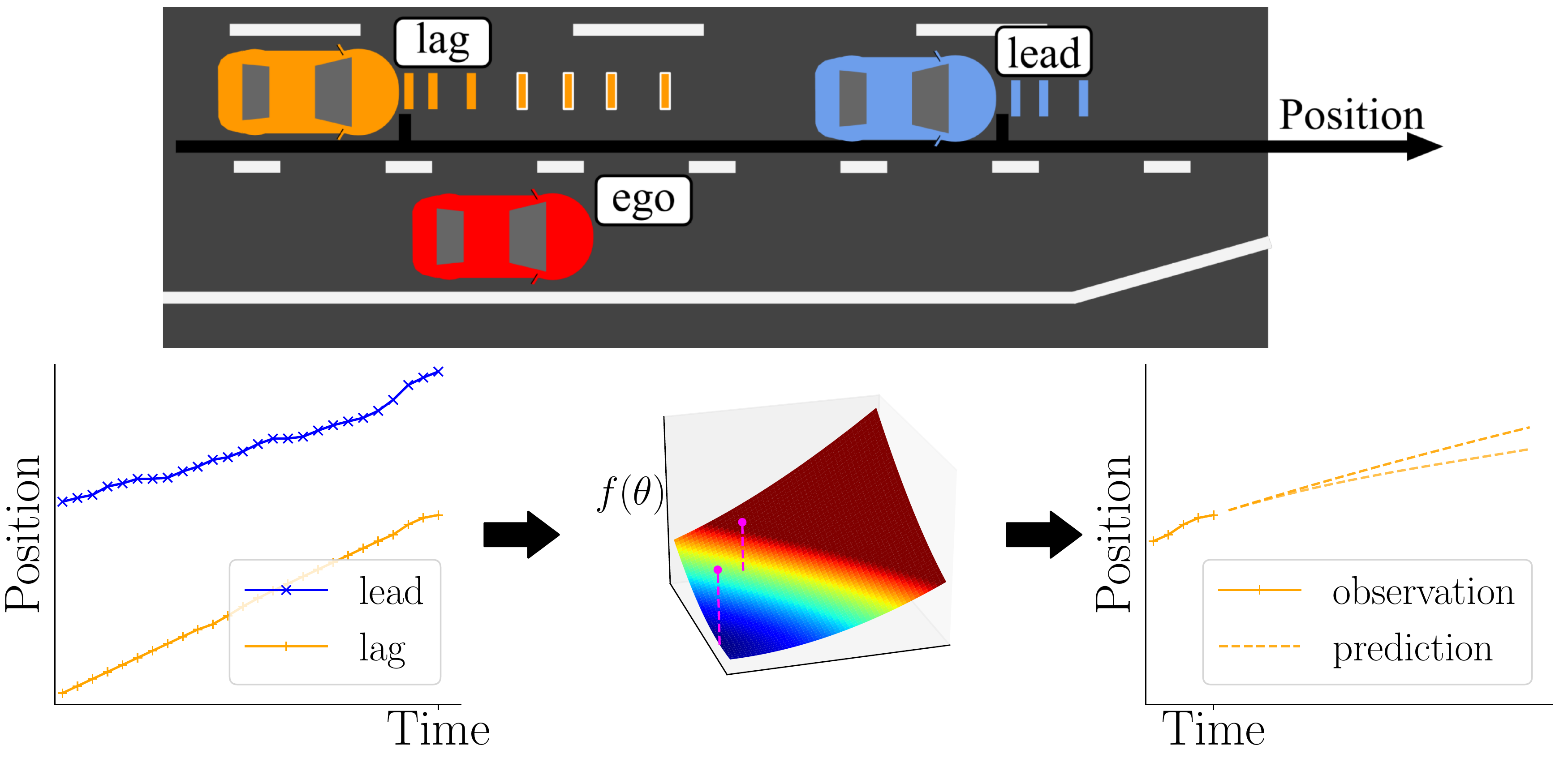}
  \caption{Overview of prediction method. The ego vehicle (red) seeks to predict the trajectories of the front and rear (lead and lag) vehicles to enable a safe merge between them. Observations of both vehicles (blue and orange marks) are used to define a likelihood function over possible controllers $\theta$ for the lag vehicle. Solving a convex problem yields an estimate that is used to sample realistic trajectories.}
  \label{fig:ramp_merging}
\end{figure}
The main contributions of this work are:
\begin{enumerate}
	\item a novel probabilistic highway interaction model;
	\item a realtime and consistent inference procedure;
	\item evaluation on merge scenarios in the real-world NGSIM dense highway traffic dataset~\cite{ngsim}.
\end{enumerate}
The proposed method achieves state-of-the-art performance with both highly constrained observation windows of \SI{400}{\milli\second} and more traditional observation windows.
We first extend the deterministic car-following model proposed by Wei et al.~\cite{wei2013auto} to the probabilistic setting.
Instead of choosing model parameters by hand, we treat them as unknown random variables and estimate them from observed velocities of the lead and lag vehicles depicted in Figure~\ref{fig:ramp_merging}.
Though the resulting estimation problem is nonconvex, we prove that it is equivalent to a semidefinite program and solve it in realtime with an off-the-shelf solver.
The estimate for the model parameters is then used to sample realistic trajectories.

The paper is organized as follows. Section~\ref{sec:related_works} describes related works in interaction-based trajectory prediction for traffic participants in general scenarios and those focused on ramp merging. Section~\ref{sec:methods} describes the interaction model we use to make predictions and the inference procedure used to determine the probabilities of different outcomes. We evaluate our model on the NGSIM dataset and perform an ablation study in Section~\ref{sec:experiments} before concluding in Section~\ref{sec:conclusion}.

\section{Related Work} \label{sec:related_works}

We first describe prediction methods that may operate on highly restricted observation windows, but do not take into account interactions. Methods that focus on modeling interactions, but operate on longer observation windows, are described in the next section. In the last section we describe methods designed specifically to account for the interactions and restricted observation windows in ramp merging scenarios. 

\subsection{Single Agent Prediction}
Classical methods specify a simple kinematic model to predict trajectories, such as constant velocity or constant yaw rate and acceleration~\cite{berthelot2011handling}.
Other methods have employed learning based approaches such as Gaussian mixture models~\cite{wiest2012probabilistic}, or hidden Markov models to ensemble simple kinematic models~\cite{kaempchen2004imm}.
More similar to the method proposed in this paper, Houenou et al.~\cite{houenou2013vehicle} combine predictions from a simple kinematic model and weight these with penalty terms on future accelerations.
These methods, however, do not account for interactions between different vehicles on the road. This can lead to inconsistent predictions in common scenarios such as a fast vehicle needing to slow down for a vehicle in front.

\subsection{General Interaction-Based Trajectory Prediction}
More recently, DNNs have been used to model the interactions between multiple vehicles.
Early works account for interaction but do not make probabilistic predictions~\cite{lee2017desire,xin2018nn}.
Treating all other agents as obstacles and predicting occupancy grids offers another approach, but this loses individual tracking labels and is limited by coarse grid size~\cite{kim2017nn}.
Directly predicting the parameters of a known distribution has been employed in most probabilistic methods, using a bivariate normal distribution~\cite{deo2018nn,hu2018nn,chandra2019traphic, ma2019trafficpredict}.
The works of~\cite{deo2018nn,hu2018nn}, however, require additional labels for maneuver types.
Another method, Traphic~\cite{chandra2019traphic}, requires no such labels, but performs a potentially slow social pooling operation for each agent at each timestep.
TrafficPredict~\cite{ma2019trafficpredict} uses an attention based mechanism to extract features for all pairwise interactions which scales poorly with the number of potentially interacting agents.
Recent works have avoided social pooling and attention mechanisms to reduce computational complexity~\cite{gupta2018socialgan,zhao2019multi}.
Social GAN~\cite{gupta2018socialgan} introduces a permutation invariant pooling layer to account for distant interactions while using a Generative Adversarial Network architecture to predict all timesteps in a single forward pass.
Multi-Agent Tensor Fusion~\cite{zhao2019multi} instead uses a global pooling layer to avoid pooling for each agent separately, as well as preserve spatial structure.
The method proposed in this paper does not account for lane changes, but for ramp merging scenarios most vehicles on the highway will cooperatively merge to an inner lane if at all merging, to avoid interacting with vehicles entering the highway~\cite{sarvi2007microsimulation,kondyli2011modeling}.

Other works not based on neural networks have relied on manually defined cost functions to specify vehicles' behavior~\cite{bahram2016combined}, or solving integer linear programs~\cite{deo2018gmm}.
These methods, however, have not performed as well as those based on neural networks.

\subsection{Ramp Merging Trajectory Prediction}
In this paper we focus on modeling the interactions between the lag vehicle and lead vehicle shown in Figure~\ref{fig:ramp_merging}.
A complementary body of research has instead focused on modeling the interactions between the lag vehicle and ego vehicle~\cite{wei2013auto,dong2018smooth,hubmann2018belief}.
In these works the goal is to infer whether or not the lag vehicle will yield to make space for the ego vehicle, or not yield. The lag vehicle is then modeled as following a controller specific to the yield intent.
The controllers have been modeled as Markov chains~\cite{dong2018smooth} and with the Intelligent Driver Model (IDM)~\cite{treiber2000congested} with manually chosen parameters~\cite{hubmann2018belief}.

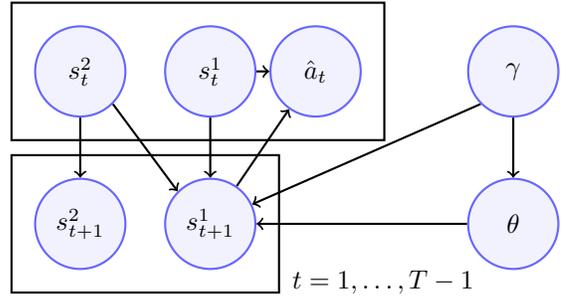
\begin{figure}[t!]
\centering
\begin{tikzpicture}[
roundnode/.style={circle, draw=blue!60, fill=blue!5, thick, minimum size=12mm, text centered},
]
\node[roundnode]      (theta)                              {$\theta$};
\node[roundnode]      (gamma) [above=0.8cm of theta]        {$\gamma$};
\node[roundnode]      (acc) [left=1.4cm of gamma]        {$\hat{a}_t$};
\node[roundnode]      (s1t) [left=2.8cm of gamma]          {$s^1_t$};
\node[roundnode]      (s1t1) [below=0.8cm of s1t]           {$s^1_{t+1}$};
\node[roundnode]      (s2t) [left=0.5cm of s1t]             {$s^2_t$};
\node[roundnode]      (s2t1) [below=0.8cm of s2t]           {$s^2_{t+1}$};

\node[draw,solid,fit=(s1t) (s2t) (acc), inner sep=0.3cm, thick] {};
\node[draw,solid,fit=(s1t1) (s2t1), inner sep=0.3cm, thick,label={[label distance=0.4mm, text=black]345:$t=1,\dots,T-1$}] {};

\draw[->, thick] (gamma.-90) -- (theta.90);
\draw[->, thick] (theta.180) -- (s1t1.0);
\draw[->, thick] (s1t.-90) -- (s1t1.90);
\draw[->, thick] (s2t.-90) -- (s2t1.90);
\draw[->, thick] (s2t.-45) -- (s1t1.135);
\draw[->, thick] (gamma.225) -- (s1t1.25);
\draw[->, thick] (s1t.0) -- (acc.180);
\draw[->, thick] (s1t1.55) -- (acc.235);
\end{tikzpicture}
\caption{Proposed interaction model for predicting lag vehicle behavior. The lag vehicle state $s^1$ depends on the lead vehicle's state $s^2$, its own controller $\theta$, and hyperparameters $\gamma$. Noisy estimates of lag vehicle acceleration $\hat{a}_t$ are calculated from state measurements.}
\label{fig:graphical_model}
\end{figure}

\section{On-Demand Trajectory Predictions} \label{sec:methods}
Here we define our model of interactions between the lag and lead vehicles. We then show how this is used to predict trajectories.
Section~\ref{ssec:problem_statement} states the trajectory prediction problem in our probabilistic setting. Section~\ref{ssec:interaction_model} defines the interaction-based controller whose parameters we aim to estimate, and the dynamics of the system. The full probabilistic model with novel regularization terms is defined in Section~\ref{ssec:regularized_prediction}. The realtime inference procedure for predicting trajectories with this model is described in Section~\ref{ssec:efficient_sampling}.

\subsection{Problem Statement}\label{ssec:problem_statement}
We consider the ramp merging scenario in Figure~\ref{fig:ramp_merging} depicting the ego vehicle merging onto the highway.
To enable safe merging we are interested in predicting the longitudinal positions of the two (lead, lag) vehicles in the target lane.
Let $s^i_t = (x^i_t~v^i_t)^\intercal \in \reals^2$ denote vehicle state, consisting of longitudinal position $x^i_t$ and velocity $v^i_t$ at timestep $t$. We denote the state of the lag vehicle by $s^1_t$ and state of the lead vehicle by $s^2_t$.
We observe the state of both vehicles $s_t = (s^{1\intercal}_t~s^{2\intercal}_t)^\intercal$ at timesteps $t=1,...,k$ and predict $x^1_t$ until the final timestep, $t=k+1,...,T$.
Additionally we will use the subscript notation $i:j$ to refer to the set of variables indexed by $i,i+1,\dots,j$.
In making probabilistic predictions this amounts to sampling
\begin{align}
    x^1_{k+1:T} \sim p(x^1_{k+1:T}|s_{1:k}).
    \label{eqn:prediction_problem}
\end{align}
We now explain our focus on the lag vehicle.
With several assumptions summarized in the graphical model shown in Figure~\ref{fig:graphical_model}, we decompose the joint prediction of both vehicles into two parts.
The first part predicts the lead vehicle's trajectory and the second predicts the lag vehicle's trajectory conditioned on that of the lead vehicle.
We start with the assumption of not having observations for the vehicle in front of the lead vehicle.
We model the lag vehicle behavior as dependent on the state of the lead vehicle, yet we do not model the state of the lead vehicle as dependent on its own lead.
One reason for this inconsistency is that occlusions and limited sensor range may prevent us from obtaining such observations. Aside from this, the proposed model could be extended to account for the lead vehicle's own lead by repeated decomposition, but there is a point at which we cannot see further vehicles ahead. We thus present the simplest model here.
Next, assume that the lead vehicle's actions do not depend on the lag vehicle's state as in common car-following models.
Furthermore, assume measurements of each vehicle's position and velocity are noiseless, while the acceleration measurement $\hat{a}_t$ of the lag vehicle has zero mean Gaussian noise with a small variance $\sigma_a^2$. This is reasonable when the former measurements have low variances but acceleration is approximated from velocity via finite differences. For example, given velocity measurements with variance $\sigma_v^2$ and timestep size $\Delta t$, acceleration then has variance $\var(\hat{a}_t) = \var(v_{t+1}-v_t)/\Delta t^2 = 2\sigma_v^2/\Delta t^2$. The small timestep will magnify the variance as in the NGSIM dataset.
Using the independence assumptions in graphical model shown in Figure~\ref{fig:graphical_model} we may write
\begin{equation}
\begin{split}
    & p(s_{k+1:T}|s_{1:k}) \\
    &= p(s^1_{k+1:T}|s_{1:k},s^2_{k+1:T})p(s^2_{k+1:T}|s_{1:k})  \\
    &= p(s^1_{k+1:T}|s_{1:k},s^2_{k+1:T})p(s^2_{k+1:T}|s^2_{1:k}) \end{split}
\end{equation}
which provides the problem decomposition. Throughout the remainder of this paper we focus on the prediction problem for the lag vehicle posed as
\begin{align} \label{eqn:prediction_problem_lag}
    s^1_{k+1:T} \sim p(s^1_{k+1:T}|s_{1:k},s^2_{k+1:T})
\end{align}
from which we obtain the predicted positions.

\subsection{Interaction Model}\label{ssec:interaction_model}
Here we describe the controller used to model the interactions between the lead and lag vehicle.
The controller is based on balancing two goals. The first is to match the speed of the lead vehicle, and the second is to maintain a desired gap to the lead vehicle.
Let $g_t$ denote the current gap between the lead and lag vehicles. This gap is calculated from their positions as $g_t = x^2_t - x^1_t - l$, where $l$ is the length of the lead vehicle. We denote the desired gap by $g_*$.
Denoting $k_v$ and $k_g$ as the proportional control gains for the desired speed and desired gap, respectively, we denote the parameters that define this controller by $\theta=(k_v~k_g~g_*)$.
The controller proposed in~\cite{wei2013auto} sets the lag vehicle's acceleration with
\begin{align} \label{eqn:controller}
    h(s_t,\theta) = k_v(v^2_t-v^1_t) + k_g(g_t-g_*)
\end{align}
according to a manually chosen $\theta$. In this work we treat $\theta$ as unknown and our main focus is to estimate it.
We assume that the parameters are nonnegative, hence $\theta \in \reals_+^3$.
We will use the notation $0_{m \times n}$ to denote the matrix of zeros with $m$ rows and $n$ columns.
Given current state of the lag vehicle and controller parameters $\theta$ the next state is given by
\begin{equation}
    s^1_{t+1} = Cs^1_t + \begin{pmatrix} \Delta t^2/2\\\Delta t \end{pmatrix} h(s_t,\theta),
\end{equation}
where
\begin{equation}
    C = \begin{pmatrix} 1&\Delta t\\0&1 \end{pmatrix}.
\end{equation}
Given the lead vehicle's states we write the entire system dynamics as
\begin{equation}
    \begin{pmatrix} s_t\\1 \end{pmatrix} =
    \begin{pmatrix} A(\theta)\\0_{2\times5}\\0_{1\times4}~1 \end{pmatrix}
    \begin{pmatrix} s_{t-1}\\1 \end{pmatrix} +
    \begin{pmatrix} 0\\0\\s^2_t\\0 \end{pmatrix},
    \label{eqn:dynamics}
\end{equation}
where
\begin{equation}
    A(\theta) = \begin{pmatrix} C&0_{2\times3} \end{pmatrix} + \begin{pmatrix} \Delta t^2/2\\\Delta t \end{pmatrix} \begin{pmatrix} -k_g\\-k_v\\k_g\\k_v\\-k_g(l+g_*) \end{pmatrix}^\intercal.
\end{equation}

\subsection{Regularized Prediction}\label{ssec:regularized_prediction}
The difficulty in sampling trajectories in~\eqref{eqn:prediction_problem_lag} stems from not knowing the lag vehicle's controller $\theta$.
Direct estimation of $\theta$ can assign significant probability to controllers that produce unrealistic behaviors.
In this section we define regularization terms to promote more realistic behaviors. 
We can express the distribution of trajectories in~\eqref{eqn:prediction_problem_lag} given the known hyperparameters $\gamma$, which we define later, as
\begin{align}
\begin{split}
    & p(s^1_{k+1:T}|s_{1:k},s^2_{k+1:T},\gamma) \\
    &= \int_{\reals_+^3} p(s^1_{k+1:T},\theta|s_{1:k},s^2_{k+1:T},\gamma) d\theta \\
    &= \int_{\reals_+^3} p(s^1_{k+1:T}|s_{1:k},s^2_{k+1:T},\theta,\gamma)
    p(\theta|s_{1:k},s^2_{k+1:T},\gamma) d\theta \\
    &= \int_{\reals_+^3} p(s^1_{k+1:T}|s_{1:k},s^2_{k+1:T},\theta,\gamma)
    p(\theta|s_{1:k},\gamma) d\theta, \end{split}
    \label{eqn:regpred0}
\end{align}
where the last equality follows from the conditional independence expressed in Figure~\ref{fig:graphical_model}.
We begin by defining the second term in~\eqref{eqn:regpred0} which can be written using Bayes' rule and calculating the accelerations derived from velocities as
\begin{align} \label{eqn:regpred_term2}
    p(\theta|s_{1:k},\gamma) \propto p(\hat{a}_{1:k-1},s_{2:k}|s_1,\theta,\gamma)p(\theta|s_1,\gamma).
\end{align}
For the first term in~\eqref{eqn:regpred_term2} we impose a recursive structure independent of $\gamma$ to mirror standard Markov chain structure as
\begin{align}
    p(\hat{a}_{1:k-1},s_{2:k}|s_1,\theta,\gamma) &=
    \prod_{i=1}^{k-1}p(\hat{a}_i,s_{i+1}|s_i,\theta),
\end{align}
which combined with the system dynamics in~\eqref{eqn:dynamics} and the assumed Gaussian noise for acceleration yields
\begin{align} \begin{split}
    & p(\hat{a}_{1:k-1},s_{2:k}|s_1,\theta,\gamma) = \prod_{i=1}^{k-1}p(\hat{a}_i|h(s_i,\theta)) =\\
    &= \prod_{i=1}^{k-1}\mathcal{N}(\hat{a}_i;h(s_i,\theta),\sigma_a^2).
     \end{split}\label{eqn:model_def1}
\end{align}
Using this term only, we could estimate $\theta$ that fits the observed data.
We now introduce the hyperparameters to address weaknesses of this initial model.
One problem is an occasionally large and unrealistic estimate of the desired gap $g_*$.
For this we regularize $g_*$ to be close to a given mean gap $g_0$.
Additionally, the car-following model was originally designed for dense traffic where the lead vehicle is near the lag vehicle.
The model is thus vulnerable to overfitting the lag vehicle's behavior to a distant lead vehicle's accelerations.
To address this, we regularize the proportional parameters to zero more as the distance between the lead and lag vehicles increases.
We introduce the scalars $\alpha,\beta$ to control the precision of the normal priors placed on the desired gap and proportional parameters.
Letting $\gamma = (\alpha, g_0, \beta)$, we define the second term in~\eqref{eqn:regpred_term2} as
\begin{align} \label{eqn:model_def2}
    -\log p(\theta|s_1,\gamma) = \alpha(g_*-g_0)^2 + \beta g_0^2(k_v^2 + k_g^2).
\end{align}
We now define the first term in~\eqref{eqn:regpred0} to regularize the future behavior of the controller, in contrast to \eqref{eqn:model_def1} which focuses on the fit to observations. Let $\chi_{\{v \preceq 0\}}(v)$ denote the characteristic function which equals zero for the real vector $v$ having all positive components and equals infinity elsewhere. The negative log-likelihood is defined to be
\begin{align} \label{eqn:model_def3}
    -\log p(s^1_{k+1:T}|s_{1:k},s^2_{k+1:T},\theta,\gamma) =
    \chi_{\{v \preceq 0\}}(v),
\end{align}
where $v=(v^1_{k+1},...,v^1_{T})^\intercal$.
Collecting the likelihoods specified in~(\ref{eqn:model_def1})-(\ref{eqn:model_def3}), the negative log-likelihood for a given $\theta$ in~\eqref{eqn:regpred0} is
\begin{align} \label{eqn:full_nll}
\begin{split}
    f(\theta) = \frac{1}{2\sigma_a^2}\sum_{i=1}^{k-1}(\hat{a}_i-h(s_i,\theta))^2 + \alpha(g_*-g_0)^2 +\\
    + \beta g_0^2 (k_v^2 + k_g^2) + \chi_{\{v \preceq 0\}}(v). &
\end{split}
\end{align}
The $\sigma_a^2$ term contributes only to the weighting of the data fit term relative to the regularization terms. We thus may set $\sigma_a^2 = 1$ for convenience and determine the other hyperparameters relative to this.
Given parameters $\theta'$ we can obtain a predicted trajectory via the dynamics given in equation~\eqref{eqn:dynamics}. This trajectory has exact probability equal to
\begin{equation}
    \frac{\exp{(-f(\theta'))}}{\int_{\reals_+^3} \exp{(-f(\theta))} d\theta}
\end{equation}
under the model. Calculating this probability by numerical integration, however, is problematic due to the large number of function evaluations.
Sampling can be used to obtain a consistent approximation, but we cannot sample $\theta$ directly from $\exp{(-f(\theta))}$ because it does not correspond to a distribution for which efficient samplers exist.
In the next section we construct a consistent and efficient sampler for the likelihood specified by~\eqref{eqn:full_nll}.

\begin{figure}[t]
  \centering
  \includegraphics[width=0.45\textwidth]{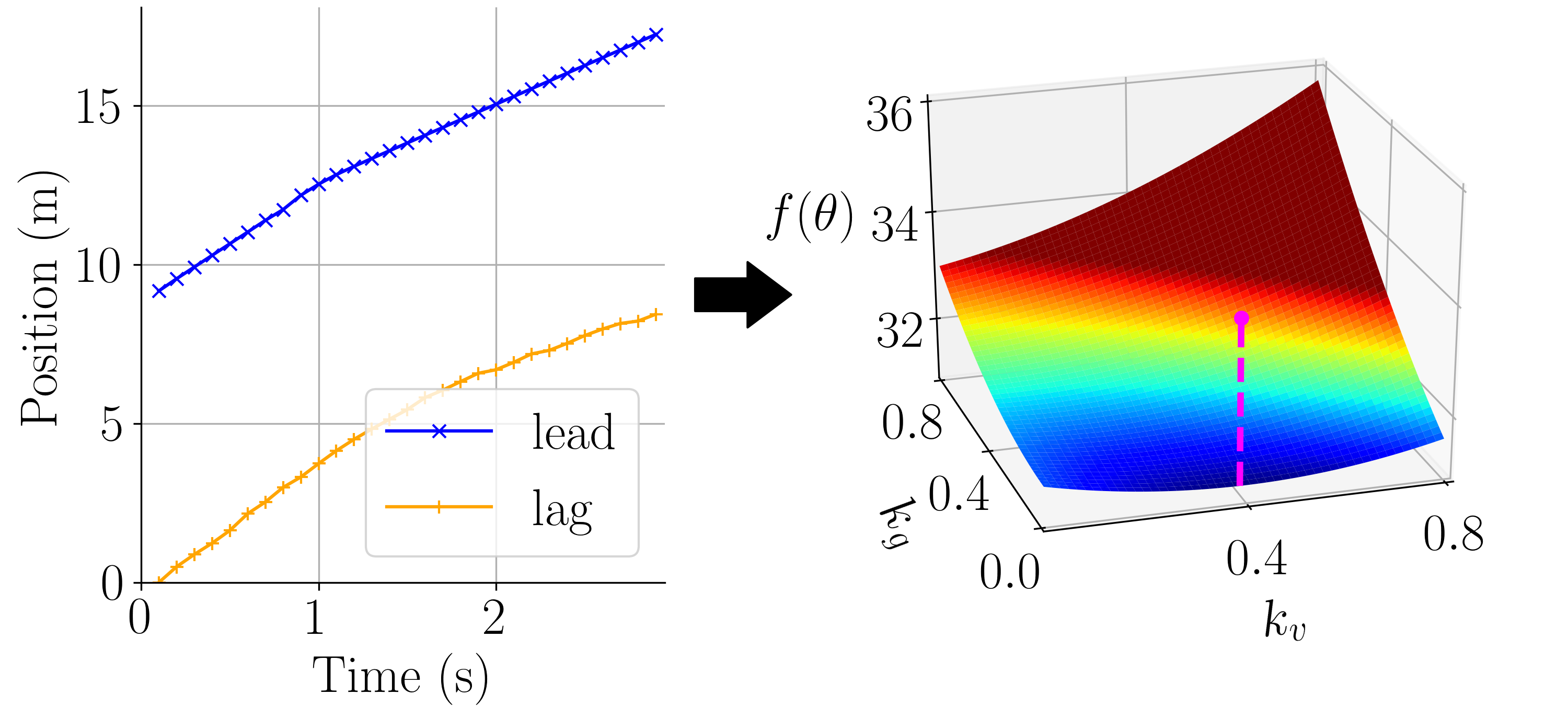}
  \caption{Likelihood surface of merging scenario. The projection $f(\theta)$ onto $k_v$ and $k_g$ appears smooth and unimodal along these parameters. The proposed inference procedure finds the global minimum (magenta) at $\hat{\theta}=(0.39, 0, 8.63)$.}
  \label{fig:surface_fig}
\end{figure}

\subsection{Efficient Sampling}\label{ssec:efficient_sampling}

Our approach to sampling from~\eqref{eqn:full_nll} has two main steps.
We first solve an optimization problem to find a set of parameters $\hat{\theta}$ that has high likelihood. We then employ importance sampling to sample from~\eqref{eqn:full_nll}.

\subsubsection{Importance Sampling Given $\hat{\theta}$ }
We sample parameters via $\theta \sim q(\theta;\hat{\theta})$ in the higher likelihood region around $\hat{\theta}$ where $q$ is a distribution chosen to have support over $\reals_+^3$ and admit efficient samplers. These samples are then weighted with their importance weights
\begin{equation}
    w = \frac{\exp{(-f(\theta))}}{q(\theta;\hat{\theta})}
\end{equation}
and normalized by the sum of the weights to complete the importance sampling.
Since $q$ has support over $\reals_+^3$, the importance sampling is consistent, and depending on the choice of $q$ this procedure may also be efficient for sampling all the high-likelihood $\theta$. Figure~\ref{fig:surface_fig} shows a typical cost surface of $f(\theta)$ using observations from NGSIM. The smooth surface and unimodality admit efficient sampling.

\subsubsection{Optimizing to Find $\hat{\theta}$}
There are multiple possible optimization problems that could be solved to find a high-likelihood $\hat{\theta}$ from $f(\theta)$. Minimizing $f(\theta)$ over $\theta \in \reals_+^3$ directly is one choice but we see that $\forall t>k ~ s_{t+1}$ depends on $\theta$ through both $A(\theta)$ and $s_t$ in~\eqref{eqn:dynamics} since we only observe up to timestep $k$.
This implies that including the future behavior regularizer~\eqref{eqn:model_def3} produces a nonconvex problem. To ensure our predictions can be made in realtime, we instead optimize over all other terms in $f(\theta)$.
Optimizing over the chosen terms yields
\begin{align}
    \begin{split}
    & \hat{\theta} = \argmin_{\theta \in \reals_+^3} \frac{1}{2} \sum_{i=1}^{k-1}(k_v(v^2_i-v^1_i) + k_g(g_i-g_*) - \hat{a}_i)^2 +\\
    &\hspace{30mm}+ \alpha(g_*-g_0)^2 + \beta g_0^2(k_v^2 + k_g^2)
    \end{split} \\
    &= \argmin_{\theta \in \reals_+^3} \frac{1}{2}||D\begin{pmatrix}k_v\\k_g\\g_*\\k_gg_*\end{pmatrix} - b||_2^2 \\
    &= \argmin_{\theta \in \reals_+^3, u \in \reals_{+}} \frac{1}{2}||D\begin{pmatrix}k_v\\k_g\\g_*\\u\end{pmatrix} - b||_2^2 ~~ \mathrm{s.t.}~~ k_gg_* = u
    \tag{NC} \label{eqn:opt_1}
\end{align}
where we collect the variables and write the optimization as a least squares problem with the rewritten known terms being $D \in \reals^{k+2,4}$ and $b \in \reals^{k+2}$. For the last equality we use a nonconvex quadratic constraint with a dummy variable to make the vector of decision variables linear.
Let $x = (\theta; u) \in \reals_+^4$.
We also rewrite the nonconvex constraint $k_gg_* = u$ in terms of $x$. Let $E \in \mathbb{S}^4$ and $c \in \reals^4$ be such that
\begin{equation} \label{eqn:r_constraint}
    r(x) = x^\intercal Ex + c^\intercal x = k_gg_* - u,
\end{equation}
whereby the constraint may be written as $r(x) = 0$.
To find an approximate solution to~\eqref{eqn:opt_1}, we remove the constraint $x\succeq0$, denoting the new problem (NC1).
By Lemma~\ref{lemma:eq} in the Appendix, solving the convex relaxation of (NC1) for $x$ given by
\begin{equation} \tag{P} \label{eqn:opt_2}
\begin{aligned}
    \underset{\substack{
        X \in \mathbb{S}^4 \\
        x \in \reals^4
    }}{\mathrm{minimize}}&
    ~~ \frac{1}{2}\tr(D^\intercal DX) - b^\intercal Dx + \frac{1}{2}b^\intercal b& \\
    \mathrm{s.t.}&~~ X \succeq xx^\intercal\\
    &~~ \tr(EX) + c^\intercal x = 0&
\end{aligned}
\end{equation}
is equivalent to solving (NC1) for $x$. The nonconvex constraint with $r(x)$ has been replaced by convex constraints.
Provided the solution satisfies $x\succeq0$, it is also the global minimum of~\eqref{eqn:opt_1}.
%
%
Moreover, the $\hat{\theta}$ found from the solution is the global minimum of the entire negative log-likelihood~\eqref{eqn:full_nll} whenever it satisfies the nonconvex constraint on future velocities~\eqref{eqn:model_def3}.
To approximate~\eqref{eqn:opt_1}, we thus solve~\eqref{eqn:opt_2} with the additional constraint $x\succeq0$, denoted (P1).
The proposed method for sampling trajectories is summarized in Algorithm~\ref{algo:sampler}.
Lemma~\ref{lemma:eq} also tells us that we need sufficiently many observations to ensure $D$ is full rank. This occurs at a minimum of two observations, and in practice we find four to be sufficient.


\begin{algorithm}[t!]
\SetAlgoLined
\KwIn{$s_{1:k}$, $s^2_{k+1:T}$, $\gamma$, $n$}  
\KwOut{$s^{1, (i)}_{k+1:T}, p^{(i)}$, for $i=1,\dots,n$}
Solve convex problem (P1) for $\hat{\theta}$  \\
\ForEach{$i = 1,...,n$}{
  Sample $\theta_i \sim q(\theta;\hat{\theta})$  \\
  Generate $s^{1, (i)}_{k+1:T}$ via~\eqref{eqn:dynamics}  \\
  $w_i \leftarrow \exp{(-f(\theta_i, s_{1:k}, s^{1, (i)}_{k+1:T}, s^2_{k+1:T}, \gamma))}/q(\theta_i;\hat{\theta})$ via~\eqref{eqn:full_nll} \\
}
$\forall i=1,\dots,n ~~ p^{(i)} \leftarrow w_i/\sum_{i=1}^n w_i$ \\
\caption{Probabilistic Trajectory Prediction for Ramp Merging}
\label{algo:sampler}
\end{algorithm}

\section{Experiments} \label{sec:experiments}

To evaluate the proposed method's ability to predict trajectories in dense traffic for ramp merging, we test it on the NGSIM dataset~\cite{ngsim}. The NGSIM dataset includes full trajectory data recorded at \SI{10}{\Hz} for two highways, I-80 and US-101, during peak usage.
Since our focus is ramp merging for AVs, we extract relevant pairs of lead and lag vehicles. These pairs are those between which a vehicle entering the highway has merged, or the pair behind such a pair.
We are most interested in predicting the behavior of the lag vehicle at the most crucial moment--when it can see the potentially merging ego vehicle.
For each pair the start of the prediction window $t=k+1$ begins when the merging vehicle first passes the lag vehicle.
The end of the prediction window $t=T$ occurs either when the ego vehicle passes the lead vehicle or first enters the target lane in case of a merge.
All pairs are observed for \SI{3.2}{\second} before the prediction window.
This choice of observation window allows us to compare to methods that use more traditional window lengths.
We extract 420 pairs from the I-80 data and 292 pairs from the US-101 data.

\begin{table*}[h]
\centering
\caption{Predictive performance of each method for the NGSIM dataset (best in \textbf{bold} and second best \underline{underlined}). Average distance error (ADE) and root mean squared error (RMSE) are shown as ADE/RMSE in meters. The proposed method achieves the lowest error for short-term predictions, and outperforms the DNNs when the observations are limited to nearby vehicles.}
\begin{tabular}{ |c|c|c|c||c|c|c|c|c| }
 \hline
    & \multicolumn{8}{c|}{\SI{3.2}{\second} observed} \\
 \hline
    & \multicolumn{3}{c||}{Extra information} & & & & & \\
  t (\si{\second})  & IDM   & SGAN*   & MATF*   & CV   & SGAN   & MATF   & Proposed-NR   & Proposed  \\
 \hline
   0.8   & 0.65/0.89 & 0.46/0.75    & 0.44/\underline{0.68}  & 0.67/0.92  & 0.67/0.96 & 0.49/0.75 & \underline{0.37}/\textbf{0.60} & \textbf{0.33}/\textbf{0.60} \\
   1.6   & 1.95/2.60 & 1.10/1.63   & \underline{1.00}/\textbf{1.44}  & 1.47/1.97    & 1.47/2.01 & 1.20/1.73 & 1.08/1.56 & \textbf{0.95}/\underline{1.62} \\
   2.4   & 3.47/4.73 & 1.87/2.60  & \textbf{1.56}/\textbf{2.17}   & 2.34/3.42    & 2.34/3.13 & 2.08/2.93 & 1.99/2.77 & \underline{1.67}/\underline{2.47} \\
   3.2   & 4.73/6.20 & 2.78/3.71 & \textbf{2.04}/\textbf{2.81} & 3.42/4.44   & 3.35/4.39    & 3.01/4.24 & 3.01/4.14 & \underline{2.54}/\underline{3.61} \\
   4.0   & 5.57/7.29 & 3.81/4.99  & \textbf{2.67}/\textbf{3.60}  & 4.63/5.91 & 4.46/5.77    & 4.19/5.87 & 4.25/5.74 & \underline{3.54}/\underline{4.88} \\
   4.8   & 5.97/7.72 & 4.90/\underline{6.26}  & \textbf{3.22}/\textbf{4.34}  & 5.94/7.60  & 5.65/7.27    & 5.42/7.52 & 5.63/7.55 & \underline{4.67}/6.31 \\
\hline
\end{tabular}
\label{tab:traditional_test}
\end{table*}

\subsection{Model Specifications}
We set $g_0$ equal to the mean of the observed gaps. For the precision values, we found that values in $[0.5, 2]$ achieve a good balance between performance and probability calibration. Following this, we set $\alpha$ and $\beta$ to 1 for all experiments.
For importance sampling we define $q(\theta;\hat{\theta})$ as the multivariate normal distribution $\mathcal{N}(\hat{\theta}, I_3)$ truncated to $\reals_+^3$ and draw 1,000 samples. This variance was found to be sufficiently large to sample effectively. 
The convex problem (P1) is solved with CVXOPT v1.2.3~\cite{cvxopt}, an open-source solver for convex optimization.

\subsection{Baselines}
We compare to state-of-the-art methods for ramp merging and general highway prediction in addition to a simplified version of the proposed model:
\begin{itemize}
    \item \textbf{Constant Velocity (CV):} The average velocity is used to predict future positions.
    \item \textbf{IDM-based (IDM)}\cite{hubmann2018belief} \textbf{:} The IDM car-following model~\cite{treiber2000congested} is parameterized based on the identified lead vehicle. Unlike other methods it uses the future trajectories of the lead vehicle and the ego vehicle.
    \item \textbf{Social GAN (SGAN)}\cite{gupta2018socialgan} \textbf{:} Shown to achieve state-of-the-art performance on NGSIM when compared to other neural networks~\cite{chandra2019traphic} despite originally being designed for joint prediction of pedestrian trajectories.
    \item \textbf{Multi-Agent Tensor Fusion (MATF)}\cite{zhao2019multi} \textbf{:} Achieves state-of-the-art performance on NGSIM using a global pooling layer to capture distant interactions while maintaining spatial structure.
    \item \textbf{No Regularization (Proposed-NR):} The proposed method without regularization terms, corresponding to only the first term of~\eqref{eqn:full_nll}.
\end{itemize}
Each DNN is trained once on each highway dataset. For making predictions on a given scenario, the model that has not seen it during training is used to make predictions. To evaluate these probabilistic predictions from SGAN and MATF we draw 100 samples.

The proposed method uses observations for the lead and lag vehicles, but SGAN and MATF have traditionally been evaluated with observations for all vehicles on the freeway~\cite{chandra2019traphic,zhao2019multi}. For comparison we include this standard evaluation, denoted by SGAN* and MATF*.
In practice, however, we will accurately detect only nearby vehicles. To reflect this case, we evaluate SGAN and MATF with the same observations as the proposed method, augmented with observations of other vehicles we may reasonably detect from the viewpoint of the ego vehicle.
We add observations for three additional vehicles: the lead vehicle's lead, the ego vehicle, and the ego's own lead vehicle.

\subsection{Evaluation Metrics}
Let $\hat{x}_{i,t}$ be the random variable corresponding to the probabilistic prediction of the lag vehicle's longitudinal position at timestep $t$ in the $i$th scenario. The true position is denoted $x_{i,t}$.
Since the time horizon varies between scenarios, we denote $N_t$ as the number of scenarios with time horizon $T \geq t$.
To evaluate the accuracy of the probabilistic trajectory predictions we evaluate the following metrics:
\begin{itemize}
    \item \textit{Average Distance Error (ADE):} The expected distance between the prediction and the true position, used in~\cite{gupta2018socialgan,chandra2019traphic,zhao2019multi,ma2019trafficpredict}. ADE is calculated at timestep $t$ as:
    \begin{align*}
        ADE(t) = \frac{1}{N_t}\sum_{i=1}^{N_t} \mathbb{E}[\lvert x_{i,t} - \hat{x}_{i,t} \rvert]
    \end{align*}
    \item \textit{Root Mean Squared Error (RMSE):} The square root of expected squared error between the prediction and the true position, used in~\cite{hu2018nn,deo2018nn,xin2018nn,chandra2019traphic,zhao2019multi}:
    \begin{align*}
        RMSE(t) = \sqrt{\frac{1}{N_t}\sum_{i=1}^{N_t} \mathbb{E}[(x_{i,t} - \hat{x}_{i,t})^2]}
    \end{align*}
\end{itemize}
The ADE tells us how prediction errors are distributed on average. Predicting a distribution over positions that has a mean close to the actual position will result in lower error.
The RMSE is similar but assigns more weight to larger errors due to the squared term within the expectation.
A method with ADE lower than RMSE suggests that it predicts more extreme cases, but assigns lower probability to these.

\begin{figure*}[t]
  \centering
  \includegraphics[width=1\textwidth]{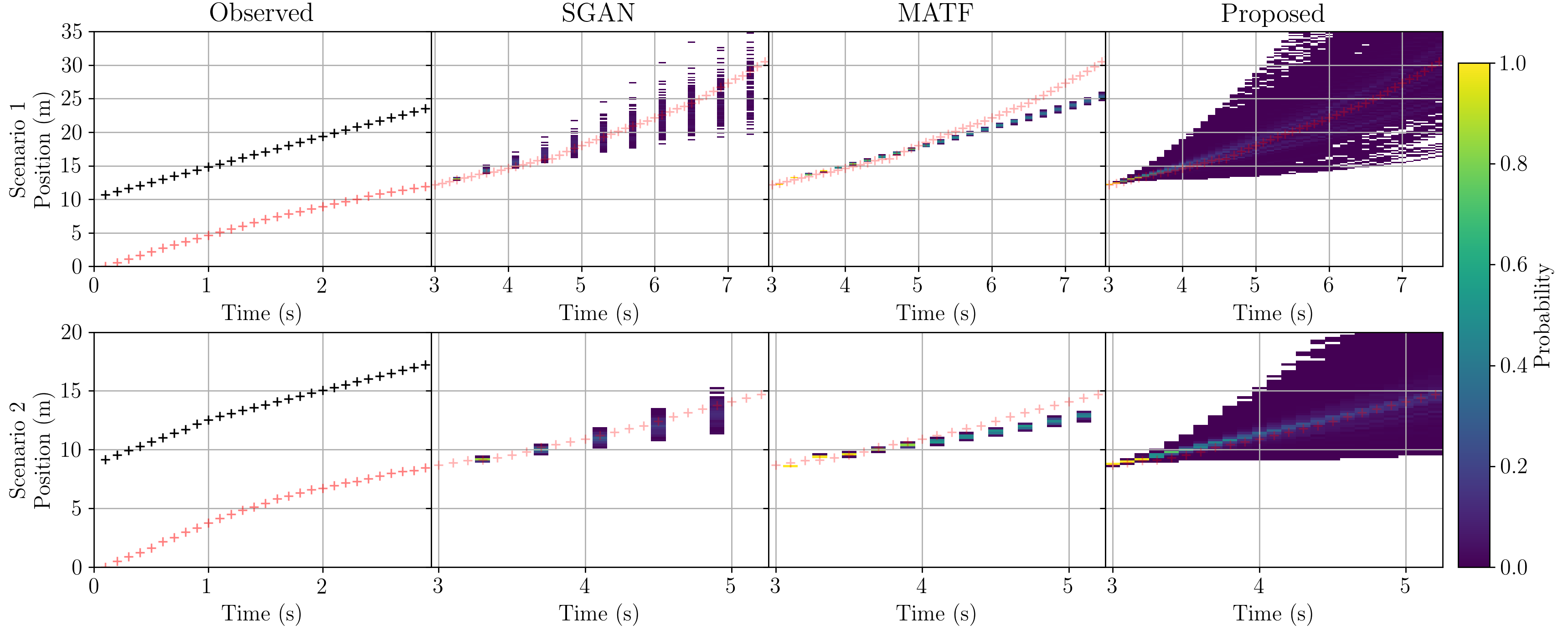}
  \caption{Predictions of each method on two ramp merging scenarios. The red crosses show the lag vehicle's position and the black crosses show the observed positions of its lead vehicle. The probabilistic predictions for each method are displayed after the end of the observation window. The color bar (right) provides the probability corresponding to each color. The proposed method predicts the lag vehicle's positions accurately despite using less information than the other interaction-based methods.}
  \label{fig:scenarios}
\end{figure*}

\subsection{Performance in Dense Traffic Scenarios}

The performance for each method is shown in Table~\ref{tab:traditional_test}.
The model-based method IDM makes overly conservative predictions about braking which hurt its performance. 
MATF* achieves the lowest errors for long-term predictions by utilizing positional information of all vehicles on the freeway. SGAN* also achieves low error with the same information.
Using the same DNNs to make predictions having observed only the more immediate vehicles, however, results in worse performance.
The proposed method is evaluated without knowledge of the lead's own lead vehicle, or the vehicles in the merge lane, yet still outperforms the DNNs.
The drop in DNN performance suggests that rather than learning to predict based on cues in nearby drivers' behavior, they have learned how traffic waves propagate along highways. 
Comparing the errors made in the short-term and the long-term, we observe that while the DNNs perform competitively with CV in the short-term, they are better at capturing behavior in the long-term.
Fitting the car-following model directly as Proposed-NR is competitive with MATF, though less so at longer-term predictions.
Adding the regularization terms enables the proposed method to outperform the baselines, excepting those with full observations of the freeway.
Even compared to these methods, the proposed method predicts the short-term with lower error. This makes sense if MATF* learned to focus on the long-term rather than the short-term, relying heavily on observations of vehicles much farther ahead.

\subsection{Performance with Limited Observations}
Previously we compared performance ensuring that each method had sufficiently many observations in each scenario.
For the real scenarios that AVs will encounter, however, we cannot guarantee that such time will be available.
Even within the NGSIM dataset alone, nearly $8\%$ of the scenarios having between \SI{400}{\milli\second} and \SI{3.2}{\second} of observations are removed from our evaluation to ensure traditional observation windows may be used.
In this section we make predictions on the same scenarios as before but limit ourselves to \SI{400}{\milli\second} of observations.
Both SGAN and MATF operate on downsampled data. SGAN operates at \SI{2.5}{\Hz}, so for the limited observation window of \SI{400}{\milli\second} it sees only one observation. To supply the method with the traditional \SI{3.2}{\second} of observations, we extrapolate using the constant velocity model from the \SI{400}{\milli\second} of original observations at \SI{10}{\Hz}. We also extrapolate for MATF which operates at \SI{5}{\Hz}.

The ADE and RMSE for a subset of the methods is shown in Table~\ref{tab:limited_observations}.
We see that the performance for both SGAN and MATF remains largely unchanged between using the traditional and limited observation windows. This has also been observed in pedestrian trajectory prediction~\cite{scholler2020constant} where only the first few observations were found to significantly affect predictive performance of a DNN.
The proposed method's performance decreases for the longer-term predictions, but still outperforms the baselines.

\begin{table}[h!]
\centering
\caption{ADE/RMSE with limited observations. CV predicts competitively in the short-term while MATF does so in the long-term. The proposed method retains its performance despite using a short observation window.}
\begin{tabular}{ |c|c|c|c|c| }
 \hline
    & \multicolumn{4}{c|}{\SI{400}{\milli\second} observed} \\
 \hline
  t (\si{\second})  & CV  & SGAN  & MATF  & Proposed \\
 \hline
   0.8   & \underline{0.43}/\underline{0.78} & 0.65/0.95 & 0.51/0.82 & \textbf{0.32}/\textbf{0.59} \\
   1.6   & \underline{1.13}/\underline{1.81} & 1.46/2.04 & 1.23/1.84 & \textbf{0.92}/\textbf{1.48} \\
   2.4   & \underline{2.01}/2.95 & 2.35/3.20 & 2.04/\underline{2.92} & \textbf{1.68}/\textbf{2.59} \\
   3.2   & 3.13/4.38 & 3.34/4.45 & \underline{2.97}/\underline{4.19} & \textbf{2.63}/\textbf{3.88} \\
   4.0   & 4.43/6.05 & 4.40/5.75 & \underline{4.09}/\underline{5.74} & \textbf{3.69}/\textbf{5.27} \\
   4.8   & 5.89/7.89 & 5.60/7.26 & \underline{5.15}/\underline{7.24} & \textbf{4.87}/\textbf{6.82} \\
\hline
\end{tabular}
\label{tab:limited_observations}
\end{table}

\begin{table}[h!]
\centering
\caption{Compute time and probability calibration. The proposed method with 1000 samples and CV predict in realtime, while SGAN and MATF with 100 samples do not. SGAN and the proposed method show calibrated probability estimates, while MATF and CV match in calibration.}
\begin{tabular}{ |c|c|c|c|c| } 
 \hline
    &  CV & SGAN & MATF & Proposed \\
 \hline
Compute time (\si{\second})       & 0.002 & 0.549 & 0.908 & 0.028 \\
Calibration    & 0.65 & 0.18   & 0.65  & 0.17  \\
\hline
\end{tabular}
\label{tab:calib}
\end{table}

\subsection{Speed}

Sudden and critical scenarios in autonomous driving demand methods that operate in realtime. Table~\ref{tab:calib} shows the time taken to make predictions for a single scenario.
SGAN and MATF are benchmarked on GTX 1080 GPU, and the other methods on Intel Core i7-6800K CPU at \SI{3.40}{\giga\Hz}. Although SGAN and MATF do not make realtime predictions with 100 samples, they could do so by reducing the number of samples. This would improve speed at the cost of less accurate probability estimates. The proposed method's low number of parameters aids in making realtime predictions despite using 1000 samples.

\subsection{Probability Calibration}

Probabilistic predictions attach a probability to each predicted outcome and enable planners to calculate risk. The calculated risk, however, depends on estimates of probability since the true distribution over future outcomes is unknown.
The constant velocity model can be viewed as a probabilistic method with degenerate probability estimates. It predicts a single outcome and assigns it full probability.
In some sense these probability estimates are not accurate, because vehicles often take different trajectories.
We measure this accuracy with a metric for the calibration of regression methods~\cite{kuleshov2018accurate}.
This measures the squared error between each confidence interval's probability and the empirical probability of outcomes within the interval being realized.
The calibration scores for each method in Table~\ref{tab:calib} mirror the predictive distributions in Figure~\ref{fig:scenarios}. SGAN and the proposed method have calibrated probability estimates, while MATF tends to underestimate the uncertainty in its predictions.

\section{Conclusion} \label{sec:conclusion}

We propose a novel probabilistic extension for a car-following model and introduce regularization terms to enforce realism in predicted behaviors. Through experiments we demonstrate that these terms lead to increased prediction accuracy for real ramp merging scenarios in dense traffic.
Comparing our model to existing methods on the NGSIM dataset shows that it achieves state-of-the-art performance. Furthermore, the proposed model maintains comparable performance when limited to using very few observations.
There are multiple limitations to the proposed model that provide grounds for future work.
The model considers interactions only between the lag and lead vehicles.
Combining this with approaches that consider interactions between the lag and ego vehicles~\cite{wei2013auto,dong2018smooth,hubmann2018belief} provides one direction for future work.
Accounting for lane changes provides another direction.




\section*{APPENDIX}

We aim to show that the nonconvex problem given in (NC1) is equivalent to the convex reformulation in~\eqref{eqn:opt_2}.
We first state the dual semidefinite program (SDP) of~\eqref{eqn:opt_2}:
\begin{equation} \tag{D} \label{eqn:opt_2_dual}
\begin{aligned}
    & \underset{\begin{subarray}{c}
        s, \mu \in \reals \\
    \end{subarray}}{\mathrm{maximize}}
    ~~ s & \\
    & \mathrm{s.t.}~~ \begin{pmatrix}\frac{1}{2}D^\intercal D + \mu E & -b^\intercal D + \mu c \\ (-b^\intercal D + \mu c)^\intercal & \frac{1}{2}b^\intercal b - s \end{pmatrix} \succeq 0 &
\end{aligned}
\end{equation}
We use the following special case of ~\cite[Theorem 6]{xia2016s}.
\begin{cor} \label{cor:one}
Let $r:\mathbb{R}^n\rightarrow\mathbb{R}$ be defined as \eqref{eqn:r_constraint}.
Suppose there exist vectors $x_1,x_2\in\reals^n$ such that $r(x_1) < 0 < r(x_2)$. If the nonconvex problem \textup{(NC1)} has value that is bounded below, the dual SDP~\eqref{eqn:opt_2_dual} always has an optimal solution $(s^*,\mu^*)$ with optimal value equal to the infimum of \textup{(NC1)}. Furthermore the infimum of \textup{(NC1)} is attained when the dual SDP possesses a feasible set that is not a singleton.
\end{cor}
\begin{proof}
This follows immediately from~\cite[Theorem 6]{xia2016s}.
\end{proof}

\begin{rmk}
Such $x_1,x_2$ can easily be found by taking $x_1 = (\frac{1}{2}, \frac{1}{2}, 2, 2)$ and $x_2 = (\frac{1}{2}, \frac{1}{2}, 2, \frac{1}{2})$, yielding $r(x_1) = (\frac{1}{2})2 - 2 < 0 < (\frac{1}{2})2 - \frac{1}{2} = r(x_2)$.
\end{rmk}
We now aim to show that the feasible $\mu$ are not unique to obtain equivalency.

\begin{lemma} \label{lemma:eq}
If $D \in \reals^{m,n}$ with $m \geq n$ as defined in~\eqref{eqn:opt_2_dual} has full rank, then the formulations  \textup{(NC1)} and~\eqref{eqn:opt_2} are equivalent and the optimal solution is attained.
\end{lemma}
\begin{proof}
First note that $D$ being full rank implies $D^\intercal D \succ 0$.
For sufficiently small $u \in \reals$, $\frac{1}{2}D^\intercal D + u I \succ 0$. For these $u$, $\frac{1}{2}D^\intercal D + \frac{u}{\lVert E\rVert_2} E \succ 0$ so the interior of $\{\mu \in \reals: \frac{1}{2}D^\intercal D + \mu E \succeq 0\}$ is nonempty. Since there exist $s \in \reals$ such that these $\mu$ are feasible for~\eqref{eqn:opt_2_dual}, and by the previous remarks, we can apply Corollary~\ref{cor:one} to obtain the desired result.
\end{proof}





\bibliographystyle{IEEEtran}
\bibliography{IEEEabrv,root}

\end{document}